\documentclass{article} % For LaTeX2e
\usepackage{collas2022_conference,times}

\usepackage{amsmath,amsfonts,bm}

\def\eqref#1{equation~\ref{#1}}
\def\1{\bm{1}}

\DeclareMathAlphabet{\mathsfit}{\encodingdefault}{\sfdefault}{m}{sl}
\SetMathAlphabet{\mathsfit}{bold}{\encodingdefault}{\sfdefault}{bx}{n}

\usepackage{hyperref}
\hypersetup{
    colorlinks=true,
    linkcolor=red,
    filecolor=magenta,      
    urlcolor=blue,
    citecolor=purple,
    pdftitle={Overleaf Example},
    pdfpagemode=FullScreen,
    }

\usepackage[noend]{algorithm,algorithmic}
\usepackage{caption}
\usepackage{amsmath, amssymb}
\usepackage[utf8x]{inputenc} 
\usepackage{hyperref}
\usepackage{interval}
\usepackage{caption}
\usepackage{hyperref}% http://ctan.org/pkg/hyperref
\usepackage{cleveref}%
\usepackage{amsthm}

\usepackage{tikz}

\definecolor{darkgray}{rgb}{0.50, 0.50, 0.50}
\definecolor{gray}{rgb}{0.70, 0.70, 0.70}
\definecolor{lightgray}{rgb}{0.92, 0.92, 0.92}

\newcommand\noisyksub{\widetilde{\mathbf{U}}^k}
\newcommand\noisykcixbir{\widetilde{\mathbf{U}}^{k-1}}
\newcommand\noisyinverse{\widetilde{\mathbf{U}}^{k^+}}
\newcommand\projnoisy{\mathcal{P}_{\widetilde{\mathbf{U}}^k}}

\newtheorem{theorem}{Theorem}

\newtheorem{lemma}[theorem]{Lemma}

\title{Adaptive Noisy Matrix Completion}

\author{Ilqar Ramazanli \\
Carnegie Mellon University\\
Pittsburgh, USA \\
\texttt{iramazan@alumni.cmu.edu}  
}

 \collasfinalcopy % Uncomment for camera-ready version, but NOT for submission.

\begin{document}

\maketitle

\begin{abstract}
Low-rank matrix completion has been studied extensively under various type of categories. 
The problem could be categorized as noisy completion or exact completion, also active or passive completion algorithms.
In this paper we focus on adaptive matrix completion with bounded type of noise.
We assume that the matrix $\mathbf{M}$ we target to recover is composed as low-rank matrix with addition of bounded small noise.
The problem has been previously studied by \cite{nina}, in a fixed sampling model.
Here, we study this problem in adaptive setting that, we continuously estimate an upper bound for the angle with the underlying low-rank subspace and noise-added subspace.
Moreover, the method suggested here, could be shown requires much smaller observation than aforementioned method.
\end{abstract}
%We extend our results to show how to adapt our algorithms to recover noisy matrices while the underlying matrix has non-degenerate random noise.

\section{Introduction}

Since Netflix announced Netflix Prize problem, low rank matrix completion has been center of attention of many researchers.
It has been observed that many datasets in the world have several eigendirections which carries main information about the column space of the underlying matrix.
Hence low-rank estimation is crucial to compress data, or use less data to represent massively sized data.\\[2ex]
One of the earliest pioneering jobs in the field of matrix completion has been done using nuclear-norm minimization.
It has been show in \cite{recht1} that 
\begin{align*}
 &   &\text{minimize}  \hspace{5mm}  &\| X \|_*  & \\
 &   &\text{subject to} \hspace{5mm} &X_{ij} = \mathbf{M}_{ij} \text{ for } (i,j) \in \Omega &
\end{align*}
could solve matrix completion problem when the size $\Omega$ is at least $\Omega(\max{\mu^2_2, \mu_1} r (m+n) \log^2(n))$.
In a follow up work \cite{recht2} have further improved this result together aligned with \cite{gross}.
In passive setting matrix completion has been mainly studied using nuclear norm minimization. 
Moreover, researchers has shown that there is a lower bound to the nuclear norm minimization, that there is infinitely many matrices that satisfies the given condition.
In particular \cite{tao} have shown that if the size of $\Omega$ is smaller than $mr\mu_0 \log{n}$ then there are infinitelty many solution to the nuclear norm minimization problem above.\\[2ex]
It has been shown in many problems, adaptive methods are outperforming passive traditional machine learning methods. 
\cite{ramazanli2022performance} has show the power of adaptivity in distribution regression problem.
\cite{paramythis2003adaptive} has studied adaptive learning for environment learning, and famous adaptive learning algorithm also proposed in \cite{riedmiller1992rprop}
Specifically for matrix completion \cite{akshay1,akshay2,ramazanliadaptive, nina} showed that for adaptivity helps us to reach theoretical bounds.\\[2ex]
Concretely, adaptive sampling helps to optimize existing matrix completion algorithms. 
In particular, \cite{haupt, ramazanli2020adaptive,kzmn} has shown that making sampling decision based on all the available existing information, rather than pre-defined sampling strategy outperforms the other method.
Its mainly consequence of making informative decision with more information always naturally over-performs decisions with less information.
One of the earliest adaptive matrix completion is due to \cite{akshay1} which tells that adaptively sampling $\mathcal{O} (r^{3/2} \log{n})$ entries for each column would be enough to recover the underlying matrix successfully.
Later, in \cite{akshay2} authors showed that this number could be reduced $\mathcal{O} (r \log^2{n})$, which later this complexity has been further optimized by \cite{poczos2020optimal} and also \cite{nina}.
Adaptive matrix completion algorithms themselves classifies to different categories, some algorithms performs in one phase \cite{akshay2, opt2}, some in two phase or more \cite{ilqarsingle}.

Low-rank matrix completion has many applications sucsh as multi-class learning, positioning of learning and gene expression analysis.
Generally speaking in any kind of problem, where computing and extracting each entry of information has high cost, low-rank completion and estimation methods comes to help to optimize the information extraction process much cheaper.

\vspace{3mm}

\section{Preliminaries}

We start this section by introducing basic notations that has been used thorough the paper.
Then, we will introduce basic basic definitions and theoretical facts those are helping us to analyze the algorithm .

In this section we start by providing notations and definitions those are used throughout the paper. 
Then we will provide a single phase \cite{nina} and multi-phase \cite{opt2} matrix recovery algorithms.\\[2ex]
Throughout the paper, we denote by $\mathbf{M}$ the target underlying $m \times n$ sized rank-$r$  matrix that we want to recover.  
$\|x\|_{p}$ denote the $L_{p}$ norm of a vector $x \in \mathbb{R}^n$.
We call $x_i$ the $i$'th coordinate of $x$. 
For any, $\Omega \subset [n]$ let $x_{\Omega}$ denote the induced subvector of $x$ from coordinates $\Omega$.
For any $\mathbf{R}\subset[m]$, $\mathbf{M}_{\mathbf{R}:}$ stands for an $|\mathbf{R}| \times n$ sized submatrix of $\mathbf{M}$ that rows are restricted by $\mathbf{R}$. 
We define $\mathbf{M}_{:\mathbf{C}}$ in a similar way for restriction with respect to columns.
Intuitively, $\mathbf{M}_{\mathbf{R}:\mathbf{C}}$ defined for $|\mathbf{R}|\times |\mathbf{C}|$ sized submatrix of $\mathbf{M}$ with rows restricted to $\mathbf{R}$ and columns restriced to $\mathbf{C}$.
Moreover, for the special case $\mathbf{M}_{i:}$ stands for $i$-th row and $\mathbf{M}_{:j}$ stands for the $j$'th column.
Similarly, $\mathbf{M}_{i:\mathbf{C}}$ will represent the restriction of the row $i$ by $C$ and $\mathbf{M}_{\mathbf{R}:j}$ represents restriction of the column $j$ by $\mathbf{R}$.
$\theta(u,v)$ stands the angle between vectors $u$ and $v$. 
Moreover, $\theta(u,\mathbb{V}) = \mathrm{min}\{\theta(u,v) | v\in \mathbb{V}\}$ and 
$\theta(\mathbb{U},\mathbb{V}) = \mathrm{max}\{\theta(u,\mathbb{V}) | u\in \mathbb{U}\}$ for subspaces $\mathbb{U}$ and $\mathbb{V}$. The projection operator to subspace $\mathbb{U}$ will be represented by $\mathcal{P}_{\mathbb{U}}$.\\[2ex]

Coherence has a very important role in all passive and adaptive matrix completion algorithms. 
It has been used to give both lower bounds for passive matrix completion algorithms and also upper bounds.
As it has been mentioned in \cite{tao} that coherence parameter of a column space / matrix is defined as

\textit{Definition:}
Coherence parameter of a matrix $\mathbf{M}$ with column space $\mathbb{U}$ is defined as following where $\mathcal{P}_{\mathbb{U}}$ represents the orthogonal projection operator onto the subspace $\mathbb{U}$. 
\begin{align*}
\mu(\mathbb{U}) = \frac{n}{r} \underset{1 \leq j \leq n}{\max} || \mathcal{P}_{\mathbb{U}} e_j ||^2,
\end{align*}

One can observe that if $e_j \in \mathbb{U} $ for some $j\in [n]$, then the coherence will attain its maximum value: $\mu(\mathbb{U}) = \frac{n}{r}$. \\

Next important observation here is given that $\mathbf{U^k}=\{u_1, u_2, \ldots, u_k\}$ and $\noisyksub=\{\widetilde{u}_1, \widetilde{u}_2, \ldots, \widetilde{u}_k\}$
To estimate the upper bound $\tilde\theta ({\mathbf{U}}^k, \noisyksub)$ for $\theta ({\mathbf{U}}^k, \noisyksub)$ we use the idea due to  \cite{blum}:
\begin{align*}
    \tilde\theta(\noisyksub,\mathbf{U}^k)  = \frac{\pi}{2} \frac{ \theta(u_k, \widetilde{u}_k) }{\theta(\widetilde{u}_k,\noisykcixbir) -\tilde\theta(\noisykcixbir, {\mathbf{U}}^{k-1} ) } + \tilde\theta(\noisykcixbir, {\mathbf{U}}^{k-1} )  
\end{align*}

\section{Main Results}

Starting here, we analyze the completion problem with the condition that entries of the underlying matrix can be noisy.
Similar to \cite{nina}, we focus on two types of noise model: sparse random noise and bounded noise. 
First, we assume that several columns of the matrix are completely noisy, and we target to recover clean entries using as little as possible observations.
We show how to extend exact completion algorithms proposed here to handle this type of noise.
Second, we assume that each entry of the underlying matrix can contain some small noise.

In this section, we propose an algorithm that gives a low-rank estimation to a matrix with small noise additional to a low-rank structure.
Specifically, we assume that the observed matrix $\mathbf{M}$ is created by adding small noise to the underlying low-rank matrix $\mathbf{L}$.
\begin{align*}
     \mathbf{M} = \mathbf{L} + \zeta \hspace{5mm} \text{such that} \hspace{5mm} \|\mathbf{L}_{:i} \|_2= 1 \hspace{3mm}  and \hspace{3mm}
     \| \zeta_{:i}\|_2 \leq \epsilon  \hspace{3mm} \forall i\in[n]
\end{align*}
The main novelty of the algorithm provided here is to decide the number of entries to be observed adaptively depending on the angle between estimated column space and actual column space.
This approach to observation complexity opens further space for future improvements.
In lemma 25 we show that the angle between estimated space and actual space cannot be too much different using similar argument to \cite{nina}, and the angle between them is upper bounded by $3\frac{\pi}{2}\sqrt{k\epsilon}$, which gives the worst observation complexity for $\mathbf{LREBN}$ with $d = \mathcal{O}(\mu(\mathbf{U})r\log^2{\frac{1}{\delta}} + mk\epsilon \log \frac{1}{\delta})$ which improves the previous rate $\mathcal{O}(\mu(\mathbf{U})r\log^2{\frac{1}{\delta}} + mk\epsilon \log^2 \frac{1}{\delta})$, especially when $\epsilon$ is relatively big that the term $mk\epsilon$ is dominating over $\mu(\mathbf{U})r\log^2{\frac{1}{\epsilon}}$.\medskip\\
Moreover, there are many cases that estimated angle is much smaller than $\sqrt{k\epsilon}$.
especially, when the basis vectors of the matrix $\mathbf{L}$ are far enough by each other (the angles between them is big enough) this quantity can be as small as $k\epsilon$, which in this case observation complexity for a given column would be $d=72 \mu(\mathbf{U}) r\log^2{\frac{1}{\delta}} + 8 m k^2\epsilon^2 \log{\frac{r}{\delta}} $ which is further smaller.
\begin{algorithm}
\caption*{  \hypertarget{lrebn}{\textbf{LREBN:}} Low-rank estimation for bounded noise. }
\textbf{Input:}   $d=72 \mu(\mathbf{U}) r\log^2{\frac{1}{\delta}} $\\
 \textbf{Initialize:}  $k=0 , \widetilde{\mathbf{U}}^0 = \emptyset, \tilde\theta(\widetilde{\mathbf{U}}^0,\mathbf{U}^0)=0$ 
\begin{algorithmic}[1]
    \STATE Draw uniformly random entries $\Omega \subset [m]$ of size $d$   
    \FOR{$i$ from $1$ to $n$}        
    \STATE  \hspace{0.2in} \textbf{if} $\| \mathbf{M}_{\Omega:i}-{\mathcal{P}_\mathbf{\widehat{U}_{\Omega}^k}} \mathbf{M}_{\Omega:i}\| >   (1+\epsilon)\Big( \sqrt{\frac{3d}{2m}} \tilde\theta ({\mathbf{U}}^k, \noisyksub) +  \sqrt{\frac{3d k \epsilon}{2m}} \Big)$
    \STATE \hspace{0.2in}  \hspace{0.2in} Fully observe $\mathbf{M}_{:i}$ 
    \STATE  \hspace{0.2in}  \hspace{0.2in}  $\widetilde{\mathbf{U}}^{k+1} \leftarrow \widetilde{\mathbf{U}}^{k} \cup \mathbf{M_{:i}} $, Orthogonalize $\widehat{\mathbf{U}}^{k+1}$
        \STATE  \hspace{0.2in}  \hspace{0.2in}  Estimate $\tilde\theta ({\mathbf{U}}^k, \noisyksub)$ the upper bound for  $\theta ({\mathbf{U}}^k, \noisyksub)$ 
        \STATE  \hspace{0.2in}  \hspace{0.2in}  $d=72 \mu(\mathbf{U}) r\log^2{\frac{1}{\delta}} + 8 m \tilde\theta(\noisyksub,\mathbf{U}^k)^2 \log{\frac{r}{\delta}} $ and set  $k=k+1$  

    \STATE \hspace{0.2in}  Draw uniformly random entries $\Omega \subset [m]$ of size $d$   

   \STATE   \hspace{0.2in}  \textbf{otherwise:} $\widetilde{\mathbf{M}}_{:i} = \widehat{\mathbf{U}}^k {\widehat{\mathbf{U}}^{k^+}_{\Omega}}
 \widetilde{\mathbf{M}}_{\Omega :i}$
\ENDFOR
\end{algorithmic}
%\algorithmicindent \textbf{Output:} return $\widehat{\mathbf{M}}$
\textbf{Output:}  $\widetilde{\mathbf{M}}$
\end{algorithm}\\

%As we discussed in the Supplementary Material, we start by $\theta(\widetilde{\mathbf{U}}^0,\mathbf{U}^0) = 0$ and for every added to basis vector

\begin{theorem}
Given the $\mathbf{L}$ be an $m\times n$ sized underlying rank-$r$ matrix where each column has $\ell_2$ norm of 1. 
Moreover, $\mathbf{M}$ is a full rank matrix where each column is created by adding at most $\ell_2$ norm$-\epsilon$ noise to the corresponding column of $\mathbf{L}$.
Then the algorithm $\mathbf{LREBN}$ estimates underlying matrix with $\ell_2$ norm of error is $\Theta(\frac{m}{d} \sqrt{k \epsilon})$ by sampling $d=72 \mu(\mathbf{U}) r\log^2{\frac{1}{\delta}} + 8 m \tilde\theta(\noisyksub,\mathbf{U}^k)^2 \log{\frac{r}{\delta}} $ entries in each column
\end{theorem}

Proofs here are inspired by the work of \cite{nina}, with the given difference that here we use different$-$ smaller observation complexity.
For sake of completeness of the proof, we prove all details here as well.\medskip\\
We first show that, estimated subspace by algorithm does not have higher dimension than $r$. 
Then we provide upper bound to the error of recovered matrix.

\begin{lemma}
Let assume that $\mathbf{M}$ is can be decomposed as rank $r$ matrix $\mathbf{L}$ with additional small noise in each column that, its $\ell_2$ norm is bounded by $\epsilon$.
Then, at the end of the termination of the algorithm $\hyperlink{lrebn}{\mathbf{LREBN}}$, estimated subspace $\noisyksub$ has dimension at most $r$.
\end{lemma}
\begin{proof}
We prove that in the execution of the algorithm, we show if a column $\mathbf{M}_{:t}$ has been detected as new column that cannot be contained in pre-selected $\noisyksub$, then $\mathbf{L}_{:t}$ is indeed cannot be contained in the $\mathbf{U}^k$.
To use triangle inequality, we notice
\begin{align*}
\theta (\mathbf{L}_{:t},\mathbf{U}^k) \geq \theta (\mathbf{L}_{:t},\widetilde{\mathbf{U}}^k) - \theta (\widetilde{\mathbf{U}}^k, {\mathbf{U}}^k )    
\end{align*}
Using the lemma \ref{ks14} we can notice that following inequalities are get satisfied:
\begin{align*}
  \| \mathbf{M}_{\Omega t} - \mathcal{P}_{\noisyksub_{\Omega}} \mathbf{M}_{\Omega:t} \| &\leq 
\sqrt{\frac{3d}{2m}} \Big( \| \mathbf{M}_{:t} - \projnoisy \mathbf{M}_{:t} \| \Big) \\[1.3ex]
&\leq \sqrt{\frac{3d}{2m}} \Big( \| \mathbf{M}_{:t}-\mathbf{L}_{:t}\| +
 \|  \mathbf{L}_{:t} - \projnoisy  \mathbf{L}_{:t} \|  \Big) +
 \| \projnoisy ( \mathbf{L}_{:t}-\mathbf{M}_{:t} )\|   \\[1.3ex]
 &\leq \sqrt{\frac{3d}{2m}} \Big( \epsilon + \theta (\mathbf{L}_{:t},\widetilde{\mathbf{U}}^k)+\epsilon \Big)
\end{align*}
From the design of the algorithm  $ \| \mathbf{M}_{\Omega t} - \mathcal{P}_{\noisyksub_{\Omega}} \mathbf{M}_{\Omega:t} \|  > (1+\epsilon)\Big( \sqrt{\frac{3d}{2m}} \theta ({\mathbf{U}}^k, \noisyksub) +  \sqrt{\frac{3d k \epsilon}{2m}} \Big)$
and using this inequality above, we conclude that
\begin{align*}
\sqrt{\frac{3d}{2m}} \theta ({\mathbf{U}}^k, \noisyksub) +  \sqrt{\frac{3d k \epsilon}{2m}}    <   \sqrt{\frac{3d}{2m}} \Big( \epsilon + \theta (\mathbf{L}_{:t},\widetilde{\mathbf{U}}^k)+\epsilon \Big)
\end{align*}
which follows that
\begin{align*}
 \theta ({\mathbf{U}}^k, \noisyksub) +  \sqrt{ k \epsilon }    <   \Big( \epsilon + \theta (\mathbf{L}_{:t},\widetilde{\mathbf{U}}^k)+\epsilon \Big)
\end{align*}
considering the fact that $\epsilon < \frac{1}{4 }$ we conclude that 
\begin{align*}
\theta (\mathbf{L}_{:t},\widetilde{\mathbf{U}}^k) \geq \theta (\mathbf{L}_{:t},\widetilde{\mathbf{U}}^k) + 2\epsilon - \sqrt{k \epsilon} > \theta ({\mathbf{U}}^k, \noisyksub) 
\end{align*}
therefore we conclude that  $\theta ({\mathbf{U}}^k, \noisyksub) < \theta (\mathbf{L}_{:t},\widetilde{\mathbf{U}}^k) $  and it follows that $\theta (\mathbf{L}_{:t},\mathbf{U}^k) > 0$. 
Moreover, one can see that after every time this inequality get satisfied, dimension of $\mathbf{U}^k$ increases by one, and considering the fact that $\mathbf{U}^k$'s are subspace of column space of $\mathbf{L}$, its dimension cannot increase more than $r$ times.
\end{proof}
Then only remaining step to provide an upper bound to recovery error.
Note that, if the algorithm decides completely observe the column, then $\ell_2$ norm of the error is upper bounded by $\epsilon$.
Then, all we need to do is to give upper bound to columns those recovered by estimated subspace.
\begin{align*}
\| \widetilde{\mathbf{M}}_{:t} -\mathbf{L}_{:t}  \| 
&= \| \noisyksub \noisyinverse_{\Omega:} \mathbf{M}_{\Omega t}- \mathbf{L}_{:t}\|  \\[1.3ex]
&\leq  \| \noisyksub \noisyinverse_{\Omega:} \mathbf{M}_{\Omega t}- \noisyksub \noisyinverse_{\Omega:} \mathbf{L}_{\Omega:t}  \| +
     \| \noisyksub \noisyinverse_{\Omega} \mathbf{L}_{\Omega:t} - \noisyksub \noisyinverse \mathbf{L}_{:t}\| +
\| \noisyksub \noisyinverse \mathbf{L}_{:t} -\mathbf{L}_{:t}\| \\[1.3ex]
&\leq  \| \noisyksub \noisyinverse_{\Omega:} ( \mathbf{M}_{\Omega t}- \mathbf{L}_{\Omega:t}) \| +
     \| \noisyksub \noisyinverse_{\Omega} \mathbf{L}_{\Omega:t} - \noisyksub \noisyinverse \mathbf{L}_{:t}\| + \sin{\theta(\mathbf{L}_{:t}, \widetilde{\mathbf{U}}^k)}  \\[1.3ex]
&\leq \| \noisyksub \noisyinverse_{\Omega:}\| \|(\mathbf{M}_{\Omega t}-\mathbf{L}_{\Omega:t}) \| +
\| \noisyksub \noisyinverse_{\Omega} \mathbf{L}_{\Omega:t} - \noisyksub \noisyinverse \mathbf{L}_{:t}\| +
\theta(\mathbf{L}_{:t}, \widetilde{\mathbf{U}}^k)
\end{align*}
Then all we need to do is to give an upper bound to the final term.
Lets start with the second term here: $\mathbf{L}_{:t} = \noisyksub v + e$ where $\noisyksub v = \noisyksub \noisyinverse \mathbf{L}_{:t}$ and note $\| e \| = \sin{\theta(\mathbf{L}_{:t}, \noisyksub)} \leq \theta(\mathbf{L}_{:t}, \noisyksub) $. Therefore:
\begin{align*}
\noisyksub \noisyinverse_{\Omega} \mathbf{L}_{\Omega:t} - \noisyksub \noisyinverse \mathbf{L}_{:t} = \noisyksub \noisyinverse_{\Omega} (\noisyksub v + e) - \noisyksub v = \noisyksub \noisyinverse_{\Omega} e
\end{align*}
Hence we conclude that:
\begin{align*}
\| \widetilde{\mathbf{M}}_{:t} -\mathbf{L}_{:t}  \| 
&\leq \| \noisyksub \noisyinverse_{\Omega:}\| \|(\mathbf{M}_{\Omega t}-\mathbf{L}_{\Omega:t}) \| +
\| \noisyksub \noisyinverse_{\Omega} e_{\Omega}\| +
\theta(\mathbf{L}_{:t}, \widetilde{\mathbf{U}}^k) \\[1.3ex]
&\leq \| \noisyksub \noisyinverse_{\Omega:}\| \|(\mathbf{M}_{\Omega t}-\mathbf{L}_{\Omega:t}) \| +
\| \noisyksub \noisyinverse_{\Omega}\|  \theta(\mathbf{L}_{:t}, \noisyksub) +
\theta(\mathbf{L}_{:t}, \widetilde{\mathbf{U}}^k). 
\end{align*}
To give upper bound to this expression, we notice  $ \| \noisyksub \noisyinverse_{\Omega:t} \| \leq \frac{ \sigma_1 (\widetilde{\mathbf{U}}^k)}{ \sigma_k(\widetilde{\mathbf{U}}^k_{\Omega:}) }\leq \Theta(\frac{m}{d})$ given the condition that $d\geq 4 \mu(\noisyksub) k \log{\frac{k}{\delta}} $ from the lemma \ref{matcher}.
From lemma \ref{noisycoh} we know that
$\mu(\noisyksub) \leq  2 \mu(\mathbf{U}^k) + 2 \frac{m}{k}\theta(\noisyksub,\mathbf{U}^k)^2$
and from lemma \ref{kcoh} we notice that $k \mu(\noisyksub) \leq r \mu(\mathbf{U})$. Then all together these facts concludes  the selected\\
$$d = 72 \mu(\mathbf{U}) r\log^2{\frac{1}{\delta}} + 8 m \theta(\noisyksub,\mathbf{U}^k)^2 \log{\frac{r}{\delta}} \geq 
8 \mu(\mathbf{U}) r\log{\frac{r}{\delta}} + 8 m \theta(\noisyksub,\mathbf{U}^k)^2 \log{\frac{r}{\delta}}$$ satisfies  $d\geq 4 \mu(\noisyksub) k \log{\frac{k}{\delta}} $ (it is assumed that $\delta \leq \frac{1}{r^{1/8}}$ ).
Therefore, we can bound $ \| \noisyksub \noisyinverse_{\Omega:t} \| $ above by 
$ \Theta(\frac{m}{d})$.\medskip\\
Now, only remaining term in the error bound above is $\theta(\mathbf{L}_{:t}, \noisyksub)$, and we use the following inequality to compare it with quantities provided as input:
\begin{align*}
    \| \projnoisy  \mathbf{M}_{:t} - \mathbf{L}_{:t}  \| \geq 
    \sin \theta( \projnoisy  \mathbf{M}_{:t} , \mathbf{L}_{:t}) \geq 
    \frac{\theta( \projnoisy  \mathbf{M}_{:t} , \mathbf{L}_{:t}) }{2} \geq
    \frac{ \theta( \noisyksub , \mathbf{L}_{:t}) }{2}
\end{align*}
and to relate the term  $ \| \projnoisy  \mathbf{M}_{:t} - \mathbf{L}_{:t}  \|$ with observed entries we again use the inequality \ref{ks14} and the fact that $\big(1+2\log{\frac{1}{\delta}} \big)^2 \leq 6 \log^2{\frac{1}{\delta}}$ once $\delta<0.1$, lemma \ref{kcoh} and lemma \ref{ededler}

\begin{align*}
    \|  \mathbf{M}_{\Omega:t}  -\projnoisy  \mathbf{M}_{\Omega:t}  \| &\geq
    \sqrt{ \frac{1}{m} \Big(\frac{d}{2} -\frac{3k \mu(\noisyksub)\beta}{2}  \Big) }  
     \|  \mathbf{M}_{:t}  -\projnoisy  \mathbf{M}_{:t}  \|    \\[1.3ex]
     &\geq \sqrt{ \frac{1}{m}\Big(\frac{d}{2} -\frac{3k \mu(\noisyksub)\beta}{2}  \Big) }  
\Big( \|\mathbf{L}_{:t}  -\projnoisy  \mathbf{M}_{\Omega:t}  \|-\| \mathbf{L}_{:t} - \mathbf{M}_{:t}\|  \Big) \\[1.3ex]
&\geq    \sqrt{ \frac{1}{m}\Big(\frac{d}{2} -\frac{3k \mu(\noisyksub)\beta}{2}  \Big) }  
\Big( \| \projnoisy  \mathbf{M}_{:t} - \mathbf{L}_{:t}  \| -\epsilon \Big) \\[1.3ex]
&\geq    \sqrt{ \frac{1}{m} \Big(\frac{d}{2} -\frac{3k \mu(\noisyksub)\beta}{2}  \Big) }  
\Big( \frac{\theta(\noisyksub, \mathbf{L}_{:t})}{2} -\epsilon \Big) \\[1.3ex]
&\geq   \sqrt{ \frac{1}{m} \Big(\frac{d}{2} -9k \mu(\noisyksub) \log^2{\frac{1}{\delta}}  \Big) }  
\Big( \frac{\theta(\noisyksub, \mathbf{L}_{:t})}{2} -\epsilon \Big) \\[1.3ex]
&\geq   \sqrt{ \frac{1}{m} \Big(\frac{d}{2} - 18 k\mu(\mathbf{U}^k) \log^2{\frac{1}{\delta}} -18m\theta(\noisyksub,\mathbf{U}^k)^2 \log^2{\frac{1}{\delta}} \Big) }
\Big( \frac{\theta(\noisyksub, \mathbf{L}_{:t})}{2} -\epsilon \Big) \\[1.3ex]
&\geq   \sqrt{ \frac{d}{2m} - \frac{18k}{m}\mu(\mathbf{U}^k) \log^2{\frac{1}{\delta}} -18\theta(\noisyksub,\mathbf{U}^k)^2 \log^2{\frac{1}{\delta}} } \Big( \frac{\theta(\noisyksub, \mathbf{L}_{:t})}{2} -\epsilon \Big) \\[1.3ex]
&\geq   \sqrt{ \frac{d}{2m} - \frac{18r}{m}\mu(\mathbf{U}) \log^2{\frac{1}{\delta}} -18\theta(\noisyksub,\mathbf{U}^k)^2 \log^2{\frac{1}{\delta}} } \Big( \frac{\theta(\noisyksub, \mathbf{L}_{:t})}{2} -\epsilon \Big) \\[1.3ex]
&\geq   \sqrt{\frac{d}{4m}} \Big( \frac{\theta(\noisyksub, \mathbf{L}_{:t})}{2} -\epsilon \Big)
\end{align*}
Hence
\begin{align*}
    \theta({\noisyksub, \mathbf{L}_{:t}})  \leq 4 \sqrt{\frac{m}{d}}   \|  \mathbf{M}_{\Omega:t}  -\projnoisy  \mathbf{M}_{\Omega:t}  \|  + 2 \epsilon  
    & \leq 4 \sqrt{\frac{m}{d}} (1+\epsilon)  \Big( \sqrt{\frac{3d}{2m}} \theta ({\mathbf{U}}^k, \noisyksub)  + \sqrt{\frac{3d k \epsilon}{2m}} + 2 \epsilon \Big)\\[1.ex]
    &\leq (1+\epsilon) \Big(\sqrt{24} \theta ({\mathbf{U}}^k, \noisyksub) +   \sqrt{8k \epsilon}  \Big)
\end{align*}
\vspace{-1mm}
also: 
\begin{align*}
\| \widetilde{\mathbf{M}}_{:t} -\mathbf{L}_{:t}  \| 
&\leq \| \noisyksub \noisyinverse_{\Omega:}\| \|(\mathbf{M}_{\Omega t}-\mathbf{L}_{\Omega:t}) \| +
\| \noisyksub \noisyinverse_{\Omega}\|  \theta(\mathbf{L}_{:t}, \noisyksub) +
\theta(\mathbf{L}_{:t}, \widetilde{\mathbf{U}}^k)  \\[1.3ex]
&\leq \frac{m}{d} \epsilon +\Big( \frac{m}{d}+1\Big) \Big(\sqrt{24} \theta ({\mathbf{U}}^k, \noisyksub) +    \sqrt{8k \epsilon}   \Big)(1+\epsilon) 
\end{align*}
Then all we need to do is to give upper bound to $\theta(\noisyksub,\mathbf{U}^k)$. 
In the proof below, we use similar argument to \cite{blum}. 
Lets assume $\mathbf{U}^k = \{u_1,u_2,\ldots, u_k\}$ and 
$\noisyksub = \{\widetilde{u}_1,\widetilde{u}_2,\ldots, \widetilde{u}_k \}$ where each of $ \| u_i-\widetilde{u}_i \| \leq \epsilon$ satisfied. 
Then using triangle inequality, lemma \ref{blum} and lemma \ref{conc}
\begin{align*}
    \theta(\noisyksub,\mathbf{U}^k) &\leq \theta(\noisyksub, \widehat{\mathbf{U}}) + \theta(\widehat{\mathbf{U}}, \mathbf{U}^k) \\[1.5ex]
    &\leq \frac{\pi}{2} \frac{ \theta(u_k, \widetilde{u}_k) }{\theta(\widetilde{u}_k,\mathbf{U}^{k-1})} + \theta(\noisykcixbir, {\mathbf{U}}^{k-1} )\\[1.5ex]
    &\leq \frac{\pi}{2} \frac{ \theta(u_k, \widetilde{u}_k) }{\theta(\widetilde{u}_k,\noisykcixbir) -\theta(\noisykcixbir, {\mathbf{U}}^{k-1} ) } + \theta(\noisykcixbir, {\mathbf{U}}^{k-1} )  \\[1.5ex]
    &\leq \frac{\pi}{2} \frac{ \theta(u_k, \widetilde{u}_k) }{\sqrt{k \epsilon} + \theta(\noisykcixbir, {\mathbf{U}}^{k-1} ) -\theta(\noisykcixbir, {\mathbf{U}}^{k-1} ) } + \theta(\noisykcixbir, {\mathbf{U}}^{k-1} )   \\[1.5ex]
&\leq    \frac{\pi}{2} \frac{\epsilon}{\sqrt{k \epsilon}} + \theta(\noisykcixbir, {\mathbf{U}}^{k-1} ) 
\end{align*}
and using lemma \ref{ind}, we can conclude that $\theta(\noisyksub, {\mathbf{U}}^{k} ) \leq \frac{3\pi}{2} \sqrt{k \epsilon}$, which gives the final bound to 
$\| \widetilde{\mathbf{M}}_{:t} -\mathbf{L}_{:t}  \|$  to be $\Theta(\frac{m}{d}\sqrt{k\epsilon})$.

\begin{lemma} \label{conc}
Given that $\| \mathbf{M}_{\Omega t} - \mathcal{P}_{\noisykcixbir_{\Omega}}   \mathbf{M}_{\Omega:t} \| \geq (1+\epsilon)\Big( \sqrt{\frac{3d}{2m}} \theta(\noisykcixbir, \mathbf{U}^{k-1}) + \sqrt{\frac{3d k \epsilon}{2m}} \Big)$ satisfies. 
Then following also satisfies: 
\begin{align*}
    \theta(\widetilde{u}_k, \noisykcixbir) \geq  \theta(\noisykcixbir, \mathbf{U}^{k-1}) + \sqrt{k \epsilon}
\end{align*}
\end{lemma}
\begin{proof} Note that simple triangle inequality implies $\| \mathbf{M}_{:t}\| \leq 1+\epsilon$

\begin{align*}
 \theta(\widetilde{u}_k, \noisykcixbir) = \theta(\mathbf{M}_{:t}, \noisykcixbir) &\geq \sin \theta(\mathbf{M}_{:t}, \noisykcixbir)  \\[1.3ex]
 &\geq\frac{\mathbf{M}_{:t}}{1+\epsilon} \sin \theta(\mathbf{M}_{:t}, \noisykcixbir) \\[1.3ex]
 &= \frac{1}{1+\epsilon} \| \mathbf{M}_{:t} - \mathcal{P}_{\noisykcixbir} \mathbf{M}_{:t} \| \\[1.3ex]
 &\geq \frac{1}{1+\epsilon} \sqrt{\frac{2m}{3d}}  \| \mathbf{M}_{\Omega t} - \mathcal{P}_{\noisykcixbir_{\Omega}}   \mathbf{M}_{\Omega:t} \| 
\end{align*}\\
Using, the fact that 
$\| \mathbf{M}_{\Omega t} - \mathcal{P}_{\noisykcixbir_{\Omega}}   \mathbf{M}_{\Omega:t} \| \geq (1+\epsilon)\Big( \sqrt{\frac{3d}{2m}} \theta(\noisykcixbir, \mathbf{U}^{k-1}) + \sqrt{\frac{3d k \epsilon}{2m}} \Big)$ we conclude $\theta(\widetilde{u}_k, \noisykcixbir) \geq  \sqrt{k \epsilon} + \theta(\noisykcixbir, {\mathbf{U}}^{k-1} )$.
\end{proof}

\begin{lemma} \label{kcoh}
Let $\mathbf{U}^k$ be a $k$-dimensional subspace of $\mathbf{U}$ which is subspace of $\mathbb{R}^m$ with dimension $r$. Then following inequality satisfied:
\begin{align*}
    k\mu(\mathbf{U}^k) \leq r \mu(\mathbf{U})
\end{align*}
\end{lemma}
\begin{proof}
\begin{align*}
 k \mu(\mathbf{U}^k) = k \frac{m}{k}  \underset{1 \leq j \leq m}{\max} \|\mathcal{P}_{\mathbf{U}^k} e_i\|^2 &=
 r \frac{m}{r}  \underset{1 \leq j \leq m}{\max}  \|\mathcal{P}_{\mathbf{U}^k} e_i\|^2 \\
&\leq  r \frac{m}{r}   \underset{1 \leq j \leq m}{\max} \|\mathcal{P}_{\mathbf{U}} e_i\|^2 \\
 &= r\mu(\mathbf{U})   
\end{align*}
and the inequality due to $\mathbf{U}^k \subseteq \mathbf{U}$
\end{proof}

\begin{lemma} \label{ind}
Let assume that $a_0 = 0$ and $a_k \leq a_{k-1} + \frac{\pi}{2} \sqrt{\frac{\epsilon}{k}}$. Then it follows that $a_k \leq \frac{3\pi}{2} \sqrt{k\epsilon}$
\end{lemma}
\begin{proof}
Its trivial to notice that $a_{1}\leq \frac{\pi}{2} \sqrt{\epsilon} \leq 3\frac{\pi}{2} \sqrt{\epsilon} $.
Lets assume by induction for a given $k$ any index $i\leq k$ satisfies $a_i \leq 3\frac{\pi}{2}\sqrt{i\epsilon}$ and then we will prove that $a_{k+1} \leq 3 \frac{\pi}{2}\sqrt{(k+1)\epsilon}$.
We prove it by contradiction, by assuming $a_{k+1} > 3 \frac{\pi}{2}\sqrt{(k+1)\epsilon}$ and conclude to a contradiction.
\begin{align*}
    a_{k+1} &> 3 \frac{\pi}{2}\sqrt{(k+1)\epsilon} \\ 
     - a_k  &\geq -3 \frac{\pi}{2}\sqrt{k\epsilon}
\end{align*}
Therefore, $a_{k+1} - a_k \geq 3\frac{\pi}{2}\sqrt{\epsilon}\Big( \sqrt{k+1}-\sqrt{k} \Big) =
3\frac{\pi}{2}\sqrt{\epsilon} \frac{1}{\sqrt{k}+\sqrt{k+1}} \geq 3\frac{\pi}{2}\sqrt{\epsilon} \frac{1}{3 \sqrt{k}} = \frac{\pi}{2} \sqrt{\frac{\epsilon}{k}}$ which contradicts to the statement of the lemma.
Therefore, assumption cannot be satisfied which follows  $a_{k+1} \leq 3 \frac{\pi}{2}\sqrt{(k+1)\epsilon}$
\end{proof}

\begin{lemma} \label{noisycoh}
Let $\noisyksub$ and $\mathbf{U}^k$ be as defined above then, coherence number of these spaces satisfies the following inequality:
\begin{align*}
\mu(\noisyksub) \leq  2 \mu(\mathbf{U}^k) + 2 \frac{m}{k}\theta(\noisyksub,\mathbf{U}^k)^2
\end{align*}
\end{lemma}
\begin{proof}
In order to achieve the goal of comparing $\mu(\noisyksub)$ and $\mu(\mathbf{U}^k)$, we first need to understand how projection to standard vectors to $\mathbf{U}^k$ differ than projection of them to $\noisyksub$:
\begin{align*}
  \|    \projnoisy e_i \| \leq  \| \mathcal{P}_{\mathbf{U}^k} e_i \| +  \| \projnoisy e_i - \mathcal{P}_{\mathbf{U}^k} e_i \|  &\leq  \| \mathcal{P}_{\mathbf{U}^k} e_i \| +  \|  \projnoisy  - \mathcal{P}_{\mathbf{U}^k}   \|  \| e_i \| \\ 
  &=  \| \mathcal{P}_{\mathbf{U}^k} e_i \| +  \sin \theta(\noisyksub, \mathbf{U}^k) \\
  &\leq   \| \mathcal{P}_{\mathbf{U}^k} e_i \| +   \theta(\noisyksub, \mathbf{U}^k) 
\end{align*}
Therefore:
\begin{align*}
\mu(\noisyksub) = \frac{m}{k} \underset{1 \leq j \leq n}{\max} \| \projnoisy e_i \|^2 &\leq \frac{m}{k} \Big( 2 \underset{1 \leq j \leq n}{\max}  \| \mathcal{P}_{\mathbf{U}^k}e_i \|^2 + 2 \theta(\noisyksub,\mathbf{U}^k)^2 \Big) \\[1.3ex] 
&= 2 \mu(\mathbf{U}^k) + 2 \frac{m}{k}\theta(\noisyksub,\mathbf{U}^k)^2
\end{align*}
\end{proof}

\begin{lemma} \label{ededler}
Lets assume the setting as discussed in the proof above. 
Then,  
\begin{align*}
    \frac{d}{4m} > \frac{18r}{m}\mu(\mathbf{U}) \log^2{\frac{1}{\delta}} + 18\theta(\noisyksub,\mathbf{U}^k)^2 \log^2{\frac{1}{\delta}}
\end{align*}
\end{lemma}
\begin{proof}
Remind $d = 72 \mu(\mathbf{U}) r\log^2{\frac{1}{\delta}} + 8 m \theta(\noisyksub,\mathbf{U}^k)^2 \log{\frac{r}{\delta}} $ and it  implies
$\frac{d}{4m} = 18 \mu(\mathbf{U}) r\log^2{\frac{1}{\delta}} + 2 m \theta(\noisyksub,\mathbf{U}^k)^2 \log{\frac{r}{\delta}} $.
Then all we need to show is $2m \log{\frac{r}{\delta}} > 18 \log^2{\frac{1}{\delta}}$.
However, we always pick $\delta$ as $9 \log{\frac{1}{\delta}} < m$, simply because $m\geq d\geq 9\log{\frac{1}{\delta}}$  
\end{proof}

\vspace{3mm}

\begin{lemma} [\bf \cite{blum}] \label{blum}
Let subspaces $U$, $V$, and $\widetilde{V}$ defined as $U = span\{ a_1,\ldots, a_{k-1}\}$, $V = span\{ a_1,\ldots, a_{k-1},b\}$ and $\widetilde{V}= span\{ a_1,\ldots, a_{k-1}, \tilde{b}\}$. 
Then following inequality satisfied:
\begin{align*}
\theta( V, \widetilde{V}) \leq \frac{\pi}{2} \frac{\theta(\tilde{b},b)}{\theta(\tilde{b},U)}    
\end{align*}

\end{lemma}

\vspace{3mm}

\begin{lemma} [\bf \cite{akshay2}] \label{ks14}
Let $\noisyksub$ be a $k$-dimensional subspace of $\mathbb{R}^m$, and set 
$d = \mathrm{max} ( \frac{8}{3}k\mu(\noisyksub)\log{\frac{2k}{\delta}}, 4 \mu \mathcal{P}_{\noisyksub}  \log{\frac{1}{\delta}})$.
Given that $\Omega$ stands for uniformly selected subset of $[m]$ then following inequality get satisfied:
$$\frac{d(1-\alpha) -  k\mu(\noisyksub ) \frac{\beta}{1-\zeta} }{m}\|y-\projnoisy y \|  \leq \| y_{\Omega}- \mathcal{P}_{\noisyksub_{\Omega}}  y_{\Omega}\| \leq (1+\alpha) \frac{d}{m} \|y- \projnoisy y \|$$
where $\alpha =  \sqrt{ 2 \frac{ \mu(\mathcal{P}_{\widetilde{\mathbf{U}}^{k^\perp}} y)}{d} \log{\frac{1}{\delta}} } + 2\frac{\mu ( \mathcal{P}_{\widetilde{\mathbf{U}}^{k^\perp}} y)} {3d} \log{\frac{1}{\delta}}$  ,  $\beta=(1+2\log{1/\delta})^2$ and $\zeta = \sqrt{ \frac{8k \mu(\widetilde{\mathbf{U}}^{k^\perp})}{3d} \log{\frac{2r}{\delta}} }$
\end{lemma}
as noted in the paper, this lemma can be used by $\alpha < 1/2$ and $\gamma < 1/3$
\vspace{3mm}
\begin{lemma}  [\bf \cite{eigen}] \label{matcher}
Consider a finite sequence $\{ \mathbf{X}_k \} \in \mathbb{R}^{n\times n} $ independent random, Hermitian matrices those satisfies: 
\begin{align*}
    0 \leq \lambda_{\mathrm{min}}(\mathbf{X}_k) \leq \lambda_{\mathrm{max}}(\mathbf{X}_k) \leq  L.
\end{align*}
Let $\mathbf{Y} = \sum\limits_k \mathbf{X}_k$ and $\mu_r$ be the $r$-th largest eigenvalue of  $\mathbb{E}[\mathbf{Y}]$ ( $\mu_r =\lambda_r ( \mathbb{E}[\mathbf{Y}])$), then for any $\epsilon \in[0,1)$ following inequality  satisfied:
\begin{align*}
    \mathbf{Pr}( \lambda_r(\mathbf{Y}) \geq (1-\epsilon) \mu_r ) \geq 1-r \Big( \frac{e^{-\epsilon}}{(1-\epsilon)^{1-\epsilon}} \Big)^{\frac{\mu_r}{L}} \geq 1-r e^{\frac{\mu_r \epsilon^2 }{2L}} 
\end{align*}
\end{lemma}

\bibliography{collas2022_conference}
\bibliographystyle{collas2022_conference}

\appendix
%\section{Appendix}

\end{document}